\documentclass[letterpaper, 10pt, conference]{ieeeconf}
\IEEEoverridecommandlockouts
\usepackage{color}
\usepackage{soul}
\usepackage{algorithm}
\usepackage{algpseudocode}
\usepackage{amsmath}
\usepackage{amssymb}
\usepackage{graphicx}
\usepackage{mathtools}
\usepackage{gensymb}
\usepackage{subfigure}
\usepackage{tikz}
\usepackage{xcolor}
\usepackage[compress]{cite}

\newcommand{\abs}[1]{\left\vert #1 \right\vert}
\DeclareMathOperator{\E}{\mathbb{E}}
\newcommand{\naturals}{\mathbb{N}}
\newcommand{\norm}[1]{\left\Vert #1 \right\Vert}
\newcommand{\Normal}{\mathcal{N}}
\newcommand{\reals}{\mathbb{R}}
\newcommand{\T}{\mathsf{T}}
\DeclareMathOperator{\Var}{Var}
\newtheorem{lemma}{Lemma}
\newtheorem{theorem}{Theorem}
\newtheorem{definition}{Definition}

\usetikzlibrary{arrows, backgrounds, decorations, positioning, shapes}
\tikzstyle{block} = [draw, rectangle, text centered, text width=5em, minimum height=3em]
\tikzstyle{arrow} = [draw, -latex]
\tikzstyle{line} = [draw]

\title{\textbf{Adaptive Leader-Follower Formation Control and Obstacle Avoidance via Deep Reinforcement Learning}}
\author{Yanlin Zhou$^{\dagger, 1}$, Fan Lu$^{\dagger, 1}$, George Pu$^{\dagger, 1}$, Xiyao Ma$^{1}$, \\Runhan Sun$^{2}$, Hsi-Yuan Chen$^{3}$, Xiaolin Li$^{4}$ and Dapeng Wu$^{1}$ 
\thanks{
$^{\dagger}$ Equal contribution.}
\thanks{
This work was supported by National Science Foundation (CNS-1624782, CNS-1747783) and Industrial Members of NSF Center for Big Learning (CBL).
}
\thanks{$^{1}$ National Science Foundation Center for Big Learning, Large-scale Intelligent Systems Laboratory, Department of Electrical and Computer Engineering, University of Florida, Gainesville, FL, USA
\texttt{zhou.y@ufl.edu};
\texttt{fan.lu@ufl.edu};
\texttt{pu.george@ufl.edu};
\texttt{maxiy@ufl.edu};
\texttt{dpwu@ieee.org}
}%
\thanks{$^{2}$ Department of Mechanical and Aerospace Engineering, University of Florida, Gainesville, FL, USA 
\texttt{runhansun@ufl.edu}; 
}%
\thanks{$^{3}$ Amazon Robotics, North Reading
\texttt{hychen@ufl.edu}}
\thanks{$^{4}$ AI Institute, Tongdun Technology
\texttt{xiaolin.li@tongdun.net}
}%
}

\begin{document}

\maketitle

\begin{abstract}
We propose a deep reinforcement learning (DRL) methodology for the tracking, obstacle avoidance, and formation control of nonholonomic robots.
By separating vision-based control into a perception module and a controller module, we can train a DRL agent without sophisticated physics or 3D modeling.
In addition, the modular framework averts daunting retrains of an image-to-action end-to-end neural network, and provides flexibility in transferring the controller to different robots.
First, we train a convolutional neural network (CNN) to accurately localize in an indoor setting with dynamic foreground/background.
Then, we design a new DRL algorithm named Momentum Policy Gradient (MPG) for continuous control tasks and prove its convergence.
We also show that MPG is robust at tracking varying leader movements and can naturally be extended to problems of formation control.
Leveraging reward shaping, features such as collision and obstacle avoidance can be easily integrated into a DRL controller.

\end{abstract}

\section{Introduction}

The traditional control problem of dynamical systems with nonholonomic constraints is a heavily researched area because of its challenging theoretical nature and its practical use. 
A wheeled mobile robot (WMR) is a typical example of a nonholonomic system. 
Researchers in the control community have targeted problems in WMR including setpoint regulation, tracking \cite{dixon2000robust}, and formation control \cite{ren2008consensus}.
Due to the nature of these problems, the control law design involves sophisticated mathematical derivations and assumptions \cite{han2019integrated}.

One of these problems is image-based localization, which involves an autonomous WMR trying to locate its camera pose with respect to the world frame \cite{piasco2018survey}.
It is an important problem in robotics since many other tasks including navigation, SLAM, and obstacle avoidance require accurate knowledge of a WMR's pose \cite{piasco2018survey}. 
PoseNet adopts a convolutional neural network (CNN) for indoor and outdoor localization \cite{kendall2015posenet}.
Using end-to-end camera pose estimation, the authors sidestep the need for feature engineering.
This method was later extended by introducing uncertainty modeling for factors such as noisy environments, motion blur, and silhouette lighting \cite{kendall2016modelling}.
Recently, there is an emerging trend of localization in dynamic indoor environments \cite{li2018indoor, lin2018deep}.

Given accurate localization methods, various vision-based control tasks such as leader following can be accomplished.
The leader-following problem is defined as an autonomous vehicle trying to follow the movement of a leader object \cite{mutz2017following}.
The form of a leader is not limited to an actual robot, but can also include virtual markers \cite{olfati2004flocking}, which can serve as a new method for controlling real robots.

However, a virtual leader is able to pass through a territory that a real follower cannot, which require the follower to perform obstacle avoidance \cite{chin1993path}.
The major challenge of applying classical control methods to the obstacle avoidance problem is having controllers correctly respond to different obstacles or unreachable areas. \cite{dixon2004adaptive,fibla2010allostatic}. A common approach is to design an adaptive or hybrid controller by considering all the cases, which is time-consuming \cite{chin1993path}.

Formation control is a leader-follower problem but with multiple leader/followers.
It includes the additional challenge of collision avoidance among group members \cite{hong2008distributed}.
Two major approaches include graph-based solutions with interactive topology and optimization-based solutions\cite{ge2018survey}.
However, the nonlinearity of nonholonomic robots adds to the challenges of modeling robots and multi-agent consensus.
Deep RL solves these problems, but training an agent requires complex environmental simulation and physics modeling of robots \cite{bruce2018learning, liu2018formation}.

\begin{figure}[h]
    \centering
    \includegraphics[width=0.9\linewidth]{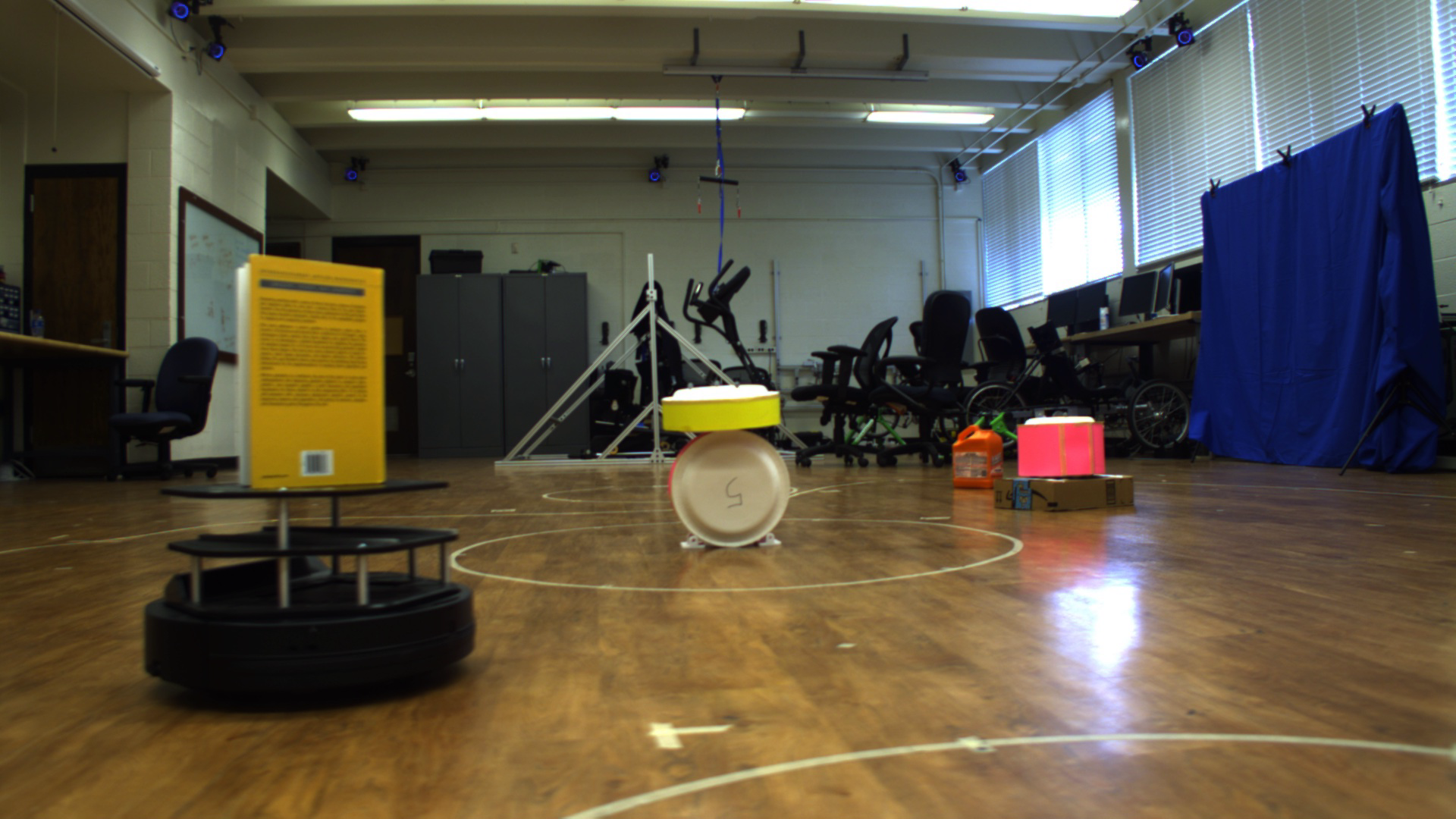}
    \caption{A screenshot from our custom dataset.}
    \label{fig:env_sc}
\end{figure}

In this paper, modular localization circumvents the challenge of not having an actual WMR's physics model or time-consuming environmental simulation.
While end-to-end training from images to control has become popular \cite{krajnik2018navigation, DBLP:journals/corr/abs-1809-10124}, partitioning the problem into two steps---localization and control---enables the flexibility of retraining the controller while reusing the localization module.
The flowchart of our modular design is shown in Figure~\ref{fig:modules}.
The first phase of this work is to apply a CNN model called residual networks (ResNets) \cite{DBLP:journals/corr/HeZRS15} to vision-based localization.
We focus on indoor localization which includes challenges such as dynamic foreground, dynamic background, motion blur, and various lighting.
The dynamic foreground and background are realized by giving intermittent views of landmarks.
A picture of our training and testing environment is shown in Figure~\ref{fig:env_sc}.
We show that our model accurately predicts the position of robots, which enables DRL without need for a detailed 3D simulation.

To replace traditional controllers that usually involve complex mathematical modeling and control law design, we leverage DRL for several control problems including leader tracking, formation control, and obstacle avoidance.
We propose a new actor-critic algorithm called Momentum Policy Gradient (MPG), an improved framework of TD3 that reduces under/overestimation \cite{fujimoto2018addressing}.
We theoretically and experimentally prove MPG's convergence and stability.
Furthermore, the proposed algorithm is efficient at solving leader-following problems with regular and irregular leader trajectories.
MPG can be extended to solve the collision avoidance and formation control problems by simply modifying the reward function.

\begin{figure}[h]
    \centering
    \begin{tikzpicture}[node distance=1cm, auto]
        \node(img){Image $I$};
        \node(cnn)[below=of img, block, fill=red!30]{ResNet};
        \node(con)[below=of cnn, block, fill=cyan!30]{Controller};
        \node(act)[below=of con]{Action $(v_x, v_y)$};
        \path[arrow] (img) -- (cnn);
        \path[arrow] (cnn) -- node{robot pose $(x, y, \theta)$} (con);
        \path[arrow] (con) -- (act);
        
        \node(sup)[left=of cnn]{Supervised learning};
        \path[line] (sup) -- (img);
        \path[line] (sup) -- (con);
        \node(rl)[left=of con]{Reinforcement learning};
        \path[line] (rl) -- (cnn);
        \path[line] (rl) -- (act);
    \end{tikzpicture}
    \caption{Information flow between modules. Blocks are neural networks.}
    \label{fig:modules}
\end{figure}
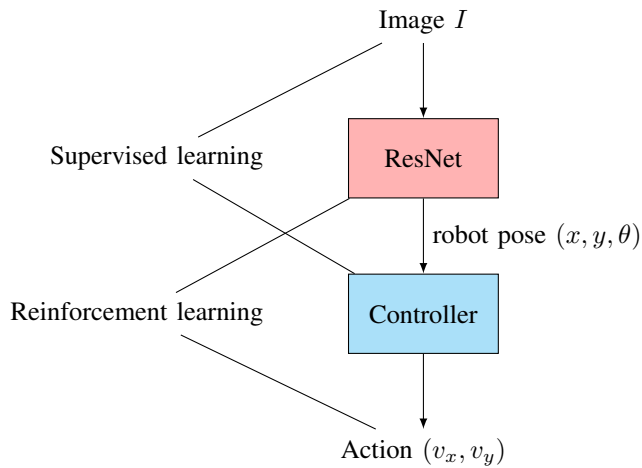

In summary, our contribution is four-fold.
\begin{itemize}
    \item We propose Momentum Policy Gradient for continuous control tasks that combats under/overestimation.
    \item A modular approach that circumvents the need for complex modelling work or control law design.
    \item Robust image-based localization achieved using CNNs.
    \item Natural extensions from the single leader-follower tracking problem to collision/obstacle avoidance and formation control with reward shaping.
\end{itemize}

\section{Mobility of Wheeled Robots}
We consider the problem of forward kinematics for WMRs that has been extensively studied over the decades.
A WMR collects two inputs: angular and linear velocities.
The velocity inputs are then fed into on-board encoders to generate torque for each wheel.
Due to the different sizes of robots, number of wheels, and moving mechanisms, robots can be classified into holonomic agents (\textit{i.e.,} omni-directional Mecanum wheel robots) \cite{ilon1975wheels} and nonholonomic agents (\textit{i.e.,} real vehicles with constrained moving direction and speed) \cite{bryant2006geometry}.
In this paper, we consider nonholonomic agents.

Ideally, angular and linear velocities are the action components of a DRL control agent.
However, these two action spaces have very different scales and usually cause size-asymmetric competition \cite{weiner1990asymmetric}.
We found out that training a DRL agent with angular and linear velocities as actions converges slower than our methods presented below.
Since there is no loss in degree of freedom, it is sufficient to use scalar velocities in $x$ and $y$ axes similar to the work done in \cite{ren2008consensus}.

Let $(x_i, y_i)$ be the Cartesian position, $\theta_i$ as orientation, and $(v_i, w_i)$ denote linear and angular velocities of the WMR agent $i$.
The dynamics of each agent is then
\begin{equation}\label{eqn:xydecomp}
    \Dot{x}_i = v_i\cos(\theta_i), \quad \Dot{y}_i = v_i\sin(\theta_i), \quad \Dot{\theta}_i = w_i
\end{equation}
The nonholonomic simplified kinematic Equation~(\ref{eqn:xynonholo}) can then be derived by linearizing (\ref{eqn:xydecomp}) with respect to a fixed reference point distance $d$ off the center of the wheel axis $(x'_i, y'_i)$ of the robot, where $x'_i = x_i + d_i\cos{\theta}, y'_i = y_i + d_i \sin{\theta}$.

Following the original setting $d = 0.15$ meters in \cite{ren2008consensus}, it is trivial to transfer velocity signals as actions used in nonholonomic system control.

\begin{equation}\label{eqn:xynonholo}
    \begin{bmatrix} v_i \\ w_i \end{bmatrix}
    =
    \begin{bmatrix}
        \cos{\theta_i} & \sin{\theta_i} \\
        -(1/d)\sin{\theta_i} &  -(1/d)\cos{\theta_i} 
    \end{bmatrix}
    \begin{bmatrix} a_{x'_i} \\ a_{y'_i} \end{bmatrix}
\end{equation}
where $a_{x'_i}$ and $a_{y'_i}$ are input control signals to each robot $i$.

In addition, other differential drive methods such as Instantaneous Center of Curvature (ICC) can be used for nonholonomic robots.
However, compared to (\ref{eqn:xynonholo}), ICC requires more physical details such as distance between the centers of the two wheels and eventually only works for two-wheeled robots \cite{dudek2010computational}.
Meanwhile, decomposing velocity dynamics to velocities on x and y axes can be reasonably applied to any WMRs.

\section{Localization}

The localization module focuses on the problem of estimating position directly from images in a noisy environment with both dynamic background and foreground.
Several landmarks (\textit{i.e.,} books, other robots) were placed so as to be visible in the foreground.
Using landmarks as reference for indoor localization tasks has proven to be successful for learning-based methods \cite{lee2018amid}.
As the WMR moves around the environment, a camera captures only intermittent view of landmarks and their position as the frame changes.

Overall, data was gathered from $3$ types of robot trajectories (\textit{i.e.,} regular, random, whirlpool) with multiple trials taking place at different times of day when the background changes greatly and lighting conditions vary from morning, afternoon, and evening.
The image dataset and ground truth pose of the agent was collected using a HD camera and a motion capture system, respectively.
As the WMR moved along a closed path, images were sampled at rate of $30$ Hz, and the ground truth pose of the camera and agent at a rate of $360$ Hz.
Data collection was performed at Nonlinear Controls and Robotics Lab.
The camera used for recording images is placed on a TurtleBot at a fixed viewing angle.
An example from our dataset is displayed in Figure \ref{fig:env_sc}.

\begin{figure}[H]
    \centering
    \includegraphics[width=0.9\linewidth]{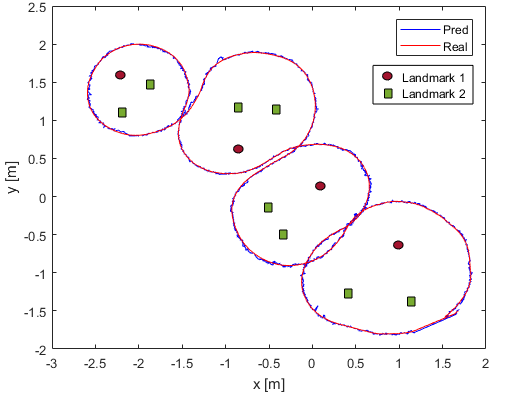}
    \caption{Example snail trajectory with predicted pose.}
    \label{fig:resnet}
\end{figure}

A ResNet-50 model \cite{DBLP:journals/corr/HeZRS15} with a resized output layer is trained from scratch on the dataset.
Residual networks contain connections between layers of different depth to improve backpropagation.
This eliminates the vanishing gradient problems encountered when training very deep CNNs.
The ResNet predicts the robots current 2D position, orientation, and the distance to nearby landmarks.
Distance to landmarks was not used in later modules, but helps focus attention onto the landmarks, which is a more reliable indicator of current pose, instead of any background changes between images.

We claim that our approach is robust enough to accurately predict a robot's pose for even very complex trajectories.
These paths can have multiple points of self-intersection.
Furthermore, ResNet based localization works well even with a dynamic foreground (multiple landmarks) and a dynamic background.
Different lighting conditions and changes in background objects between trials do not affect the accuracy.
The error rates are given in Table \ref{tab:res-error} for the robot's 2D coordinates and 4D quaternion orientation.
They are all less than $1\%$.
Figure \ref{fig:resnet} shows an example predicted and motion captured poses as well as landmarks; there is almost no difference between the two. 

\begin{table}[h]
\centering
    \caption{ResNet52 Pose Predition Error}
    \label{tab:res-error}
    \begin{tabular}{|c||c||c||c||c||c|}
        \hline
        x (\%) & y (\%) & q1 (\%) & q2 (\%) & q3 (\%) & q4 (\%) \\ 
        \hline
        $0.439$ & $0.1191$ & $0.7462$ & $0.3199$  & $0.1673$  & $0.152$    \\ 
        \hline
    \end{tabular}
\end{table}

\section{Momentum Policy Gradient}

\begin{algorithm}[H]
    \begin{algorithmic}[1]
    \caption{Momentum Policy Gradient}\label{alg:mpg}
    \State Initialize critic networks $Q_{\theta_1}$, $Q_{\theta_2}$, and actor network $\pi_\phi$ with random parameters $\theta_1, \theta_2, \phi$
    \State Initialize target networks $Q_{\hat{\theta}_1}, Q_{\hat{\theta}_2}$ with $\hat{\theta}_1 = \theta_1, \hat{\theta}_2 = \theta_2$
    \State Create empty experience replay buffer $E$
    \For{$t = 1$ \textbf{to} $T$}
        \State Take action $a = \pi_\phi(s) + \Normal(0, v_{explore})$
        \State $v_{explore} = \min(\lambda \cdot v_{explore}, v_{min})$
        \State Get next state $s'$, reward $r$ from environment
        \State Push $(s, a, s', r)$ into $E$
        \State Sample mini-batch of $N$ transitions from $E$
        \State Set $\Delta_{last} = 0$
        \For{$i = 1$ \textbf{to} $I$}
            \State $a' \gets \pi_\phi(a) + \Normal(0, v_{train})$
            \State $\Delta_{adj} \gets \frac{1}{2} (\Delta_{last} + \abs{Q_{\hat{\theta}_1}(s', a') - Q_{\hat{\theta}_2}(s', a')})$
            \State $\Delta_{last} \gets \abs{Q_{\hat{\theta}_1}(s', a') - Q_{\hat{\theta}_2}(s', a')}$
            \State $q \gets \max(Q_{\hat{\theta}_1}(s', a'), Q_{\hat{\theta}_2}(s', a')) - \Delta_{adj}$
            \State $y \gets r + \gamma q$
            \State $L^C \gets \frac{1}{N} \sum_{batch} \sum_{i=1,2} (Q_{\theta_i}(s', a') - y)^2$
            \State Minimize $L^c$
            \If{$i \bmod F = 0$}
                \State $L^A \gets$ average of $-Q_{\theta_1}(s, \pi_\phi(a))$
                \State Minimize $L^a$
                \State $\theta_1 \gets \tau\theta_1 + (1 - \tau)\hat{\theta}_1$
                \State $\theta_2 \gets \tau\theta_2 + (1 - \tau)\hat{\theta}_2$
            \EndIf
        \EndFor
    \EndFor
    \end{algorithmic}
\end{algorithm}

Since nonholonomic controllers have a continuous action space, we design our algorithm based on the framework established by DDPG \cite{lillicrap2015continuous}.
There are two neural networks: a policy network predicts the action $a = \pi_\phi(s)$ given the state $s$, a Q-network $Q_\theta$ estimates the expected cumulative reward for each state-action pair.
\begin{equation}
    Q_\theta(s, a) \approx \E\left[ \sum_{t=0}^{\infty} \gamma^t r_t \bigg| s_0 = s, a_0 = a \right]
\end{equation}
The Q-network is part of the loss function for the policy network.
For this reason, the policy network is called the actor and the Q-network is called the critic.
\begin{equation}
    L^A = Q_\theta(s, \pi_\phi(s))
\end{equation}
The critic itself is trained using a Bellman equation derived loss function.
\begin{equation}
    L^C = ( Q_\theta(s, a) - [ r + \gamma Q_{\hat{\theta}}(s', \pi_\phi(s')) ] )^2
\end{equation}
However, this type of loss leads to overestimation of the true total return \cite{franccois2018introduction}.
TD3 fixes this by using two Q-value estimators $Q_{\theta_1}, Q_{\theta_2}$ and taking the lesser of the two \cite{fujimoto2018addressing}.
\begin{equation}
    y = r + \gamma \min_{i=1,2} Q_{\theta_i} (s', \pi_{\phi_1}(s'))
\end{equation}
Note that this is equivalent to taking the maximum, and then subtracting by the absolute difference.
However, always choosing the lower value brings underestimation and higher variance \cite{fujimoto2018addressing}.

To lower the variance in the estimate, inspired by the momentum used for optimization in \cite{zhang2015deep}, we propose Momentum Policy Gradient illustrated in Algorithm~\ref{alg:mpg} which averages the current difference with the previous difference $\Delta_{last}$.
\begin{align}
    q &= \max(Q_{\hat{\theta}_1}(s', a'), Q_{\hat{\theta}_2}(s', a')) - \Delta_{adj} \\
    \Delta_{adj} &= \frac{1}{2} \left( \abs{Q_{\theta_1}(s', a') - Q_{\theta_2}(s', a')} + \Delta_{last} \right)
\end{align}
This combats overestimation bias more aggressively than just taking the minimum of $Q_{\theta_1}, Q_{\theta_2}$.
Moreover, this counters any over-tendency TD3 might have towards underestimation.
Because neural networks are randomly initialized, by pure chance $\abs{Q_{\theta_1}(s', a') - Q_{\theta_2}(s', a')}$ could be large.
However, it is unlikely that $\Delta_{last}$ and $\abs{Q_{\theta_1}(s', a') - Q_{\theta_2}(s', a')}$ are both large as they are computed using different batches of data.
Thus $\Delta_{adj}$ has a lower variance than $\abs{Q_{\theta_1}(s', a') - Q_{\theta_2}(s', a')}$.

In the case of negative rewards, the minimum takes the larger $Q_{\theta_i}(s', a')$ in magnitude.
This will actually encourage overestimation (here the estimates trend toward $-\infty$).

Before proving the convergence of our algorithm, we first require a lemma proved in  \cite{singh2000convergence}.

\begin{lemma}
	Consider a stochastic process $(\alpha_t, \Delta_t, F_t), t \in \naturals$ where $\alpha_t, \Delta_t, F_t\colon X \to \reals$ such that
	\begin{equation*}
		\Delta_{t+1}(x) = [ 1 - \alpha_t(x) ] \Delta_t(x) + \alpha_t(x) F_t(x)
	\end{equation*}
	for all $x \in X, t \in \naturals$. Let $(P_t)$ be a sequence of increasing $\sigma$-algebras such that $\alpha_t, \Delta_t, F_{t-1}$ are $P_t$-measurable.
	If
	\begin{enumerate}
		\item the set $X$ is finite
		\item $0 \leq \alpha_t(x_t) \leq 1$, $\sum_t \alpha_t(x_t) = \infty$, but $\sum_t \alpha_t^2(x_t) < \infty$ with probability 1
		\item $\norm{\E[F_t | P_t]} \leq \kappa\norm{\Delta_t} + c_t$ where $\kappa \in [0, 1)$ and $c_t$ converges to 0 with probability 1
		\item $\Var[F_t(x) | P_t] \leq K(1 + \norm{\Delta_t})^2$ for some constant $K$,
	\end{enumerate}
	Then $\Delta_t$ converges to 0 with probability 1.
\end{lemma}

The theorem and proof of MPG's convergence is borrowed heavily from those for Clipped Double Q-learning \cite{fujimoto2018addressing}.

\begin{theorem}[Convergence of MPG update rule]
	Consider a finite MDP with $0 \leq \gamma < 1$ and suppose
	\begin{enumerate}
		\item each state-action pair is sampled an infinite number of times
		\item Q-values are stored in a lookup table
		\item $Q, Q'$ receive an infinite number of updates
		\item the learning rates satisfy $0 < \alpha_t(s_t, a_t) < 1$, $\sum_t \alpha_t(s_t, a_t) = \infty$, but $\sum_t \alpha_t^2(s_t, a_t) < \infty$ with probability 1
		\item $\Var[r(s, a)] < \infty$ for all state-action pairs.
	\end{enumerate}
	Then Momentum converges to the optimal value function $Q^*$.
\end{theorem}

\begin{proof}
	Let $X = \mathcal{S} \times \mathcal{A}$, $\Delta_t = Q_t - Q^*$, and $P_t = \{Q_k, Q'_k, s_k, a_k, r_k, \alpha_k\}_{k=1}^t$.
	Conditions 1 and 2 of Lemma 1 are satisfied. By the definition of Momentum Policy Gradient,
	\begin{align*}
		\Delta^{adj}_t &= \frac{1}{2}( |Q_t(s_t, a_t) - Q'_t(s_t, a_t)| + \\
		&\qquad |Q_{t-1}(s_{t-1}, a_{t-1}) - Q'_{t-1}(s_{t-1}, a_{t-1})| ) \\
		y_t &= r_t + \gamma ( m_t - \Delta^{adj}_t ) \\
		Q_{t+1}(s_t, a_t) &= [ 1 - \alpha_t(s_t, a_t) ] Q_t(s_t, a_t) + \alpha_t(s_t, a_t) y_t
	\end{align*}
	where $m_t = \max\{ Q_t(s_t, a_t), Q'_t(s_t, a_t) \}$.
	Then
	\begin{align*}
		\Delta_{t+1}(s_t, a_t) &= [ 1 - \alpha_t(s_t, a_t) ] [ Q_t(s_t, a_t) - Q^*(s_t, a_t) ] \\
		&\quad \ + \alpha_t(s_t, a_t) [ y_t - Q^*(s_t, a_t) ] \\
		&= [ 1 - \alpha_t(s_t, a_t) ] [ Q_t(s_t, a_t) - Q^*(s_t, a_t) ] \\
		&\quad \ + \alpha_t(s_t, a_t) F_t(s_t, a_t)
	\end{align*}
	where
	\begin{align*}
		F_t(s_t, s_t) &= y_t - Q^*(s_t, a_t) \\
		&= r_t + \gamma ( m_t - \Delta^{adj}_t ) - Q^*(s_t, a_t) \\
		&= r_t + \gamma ( m_t - \Delta^{adj}_t ) - Q^*(s_t, a_t) \\
		&\quad \ + \gamma Q_t(s_t, a_t) - \gamma Q_t(s_t, a_t) \\
		&= F^Q_t(s_t, a_t) + \gamma b_t
	\end{align*}
	We have split $F_t$ into two parts: a term from standard Q-learning, and $\gamma$ times another expression.
	\begin{align*}
	    F^Q_t(s_t, a_t) &= r_t + \gamma Q_t(s_t, a_t) - Q^*(s_t, a_t) \\
	    b_t &= m_t - \Delta^{adj}_t - Q_t(s_t, a_t)
	\end{align*}
	As it is well known that $\E\left[ F^Q_t | P_t \right] \leq \gamma\norm{\Delta_t}$, condition (3) of Lemma 1 holds if we can show $b_t$ converges to 0 with probability 1.
	Let $\Delta'_t = Q_t - Q'_t$. If $\Delta_t'(s_t, a_t) \to 0$ with probability 1, then
	\begin{equation*}
		m_t \to Q_t(s_t, a_t), \qquad \Delta^{adj}_t \to 0
	\end{equation*}
	so $b_t \to 0$.
	Therefore showing $\Delta'_t(s_t, a_t) \to 0$ proves that $b_t$ converges to 0.
	\begin{align*}
		\Delta'_{t+1}(s_t, a_t) &= Q_{t+1}(s_t, a_t) - Q'_{t+1}(s_t, a_t) \\
		&= [ 1 - \alpha_t(s_t, a_t) ] Q_t(s_t, a_t) + \alpha_t(s_t, a_t) y_t - \\
		&\quad \ ( [ 1 - \alpha_t(s_t, a_t) ] Q'_t(s_t, a_t) + \alpha_t(s_t, a_t) y_t ) \\
		&= [ 1 - \alpha_t(s_t, a_t) ] \Delta'_t(s_t, a_t)
	\end{align*}
	This clearly converges to 0. Hence $Q_t$ converges to $Q^*$ as $\Delta_t(s_t, a_t)$ converges to 0 with probability $1$ by Lemma 1.
	The convergence of $Q'_t$ follows from a similar argument, with the roles of $Q$ and $Q'$ reversed.
\end{proof}

\section{Continuous Control}

\begin{figure*}[h]
    \centering
    \subfigure[Test on random]{
    \includegraphics[width=0.3\textwidth]{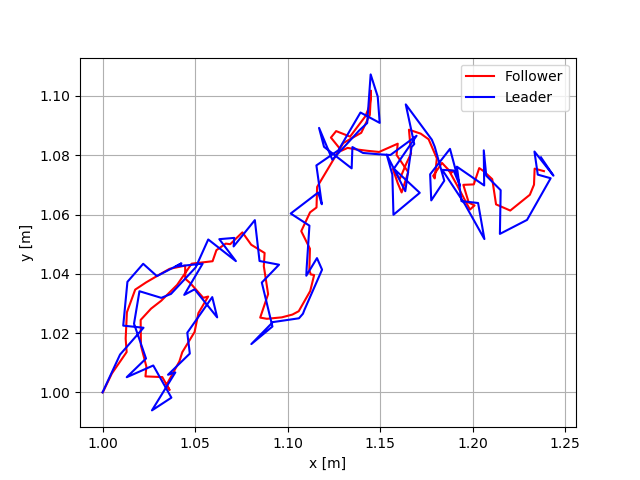}}
    \subfigure[Test on square]{
    \includegraphics[width=0.3\textwidth]{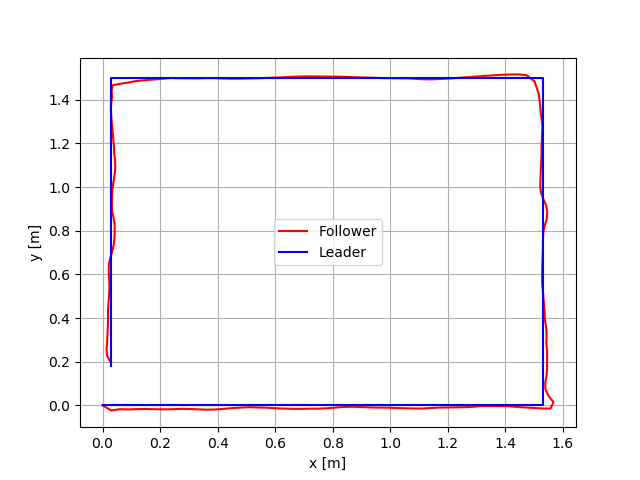}}
    \subfigure[Test on circle]{
    \includegraphics[width=0.3\textwidth]{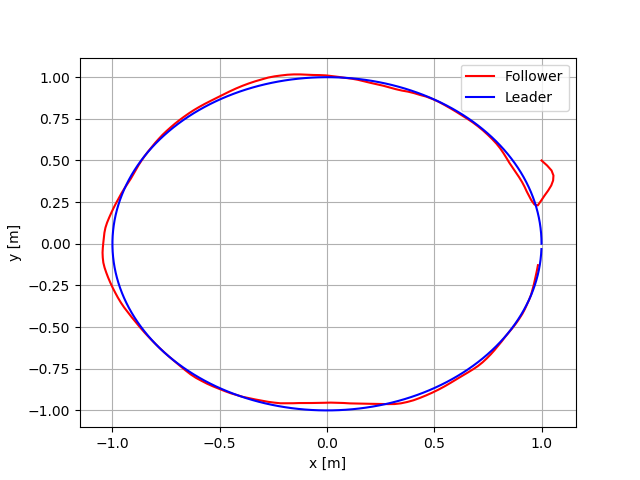}}
    \caption{Performance of a follower trained at 100 episodes with random leader on (a) random trajectory, (b) square leader, (c) circle leader.}
    \label{fig:ltperform}
\end{figure*}

We demonstrate MPG on a variety of leader-follower continuous control tasks.
In all simulations, neural networks have two hidden layers of $400$ and $300$ units respectively.
Output and input sizes vary depending on the environment and purpose of the model.
Constraints on the motion of robots (Table \ref{tab:kin_const}) are enforced by ending episodes once they are breached.
A penalty of $-k_b$ is also added to the reward for that time step.
The hyperparameters of MPG are given in Table~\ref{tab:hyper}.

\begin{table}[h]
\centering
    \caption{Kinematic Constraints of Agents}
    \label{tab:kin_const}
    \begin{tabular}{|c||c||c||c||c||c||c|}
        \hline
                 & $x_{min}$ & $x_{max}$ & $y_{min}$ & $y_{max}$ & $v_{min}$ & $v_{max}$ \\
        \hline
        Leader   & $-1$      & $1$       & $-1$      & $1$       & $-0.7$    & $0.7$     \\
        \hline
        Follower & $-2$      & $2$       & $-2$      & $2$       & $-0.7$    & $0.7$     \\
        \hline
    \end{tabular}
\end{table}

\begin{table}[h]
    \centering
    \caption{MPG Hyperparameters}
    \label{tab:hyper}
    \begin{tabular}{|l||c||c|}
        \hline
        \textbf{Hyper-parameter} & \textbf{Symbol} & \textbf{Value} \\
        \hline
        Actor Learning Rate      & $\alpha$        & $10^{-3}$      \\
        \hline
        Critic Learning Rate     & $\alpha_C$      & $10^{-2}$      \\
        \hline
        Batch Size               & $-$             & $16$           \\
        \hline
        Discount factor          & $\gamma$        & $0.99$         \\
        \hline
        Number of steps in each episode & $-$      & $200$          \\
        \hline
        Training noise variance  & $v_{train}$     & $0.2$          \\
        \hline
        Initial exploration noise variance & $v_{explore}$ & $2$    \\
        \hline
        Minimum exploration noise variance & $v_{min}$ & $0.01$     \\
        \hline
        Exploration noise variance decay rate & $\lambda$ & $0.99$  \\
        \hline
    \end{tabular}
\end{table}

Suppose there are $N$ agents whose kinematics are governed by Equations \ref{eqn:xydecomp} and \ref{eqn:xynonholo}.
Let $z_i = [\Dot{p}_i \ \ddot{p}_i]^\T$ denote the state of agent $i$ and the desired formation control for agents follow the constraint \cite{oh2015survey}:
\begin{equation}\label{eqn:fc}
    F(z) = F(z^*)
\end{equation}

\begin{definition}\label{def:fc}
    From (\ref{eqn:fc}), we consider displacement-based formation control with the updated constraint given as:
    \begin{equation}
        F(z) \coloneqq [...(z_j - z_i)^\T...]^\T = F(z^*)
    \end{equation}
\end{definition}

Each agent measures the position of other agents with respect to a global coordinate system.
However, absolute state measurements with respect to the global coordinate system is not needed.
A general assumption made for formation control communication is that all agent's position, trajectory or dynamics should be partially or fully observable \cite{Zegers.Chen.ea2019, han2019integrated}. 

The leader-follower tracking problem can be viewed as formation control with only $N = 2$ agents.

\subsection{Tracking}

We first test MPG by training a follower agent to track a leader whose position is known.
The leader constantly moves without waiting the follower.
The follower agent is punished based on its distance to the leader and rewarded for being very close.
Let $p_l$ be the 2D position of the leader and $p_f$ be the corresponding position for the follower.
The reward function for discrete leader movement is defined as 
\begin{equation}\label{eqn:wait}
    r = -\norm{p_l - p_f}
\end{equation}
where $\norm{\cdot}$ is the L2 norm.

Then, We train a follower to track a leader moving in a circular pattern.
Remarkably, this follower: can generalize to a scaled-up or scaled-down trajectory, is robust to perturbations in the initial starting location, and even tracks a leader with an elliptical motion pattern that it has never encountered before.
However, the follower fails to track a square leader which is caused by under exploration of the entire environment.
The under exploration issue is further highlighted when training a follower to track a leader with square motion.
The model learns to go straight along an edge in less than $10$ episodes, but fails to learn to turn a $90\degree$ angle for around $50$ episodes.

Agents trained against a random moving leader generalize well to regular patterns.
\begin{equation}
    v^l_x, v^l_y \sim \Normal(0, 1)
\end{equation}
This is very similar to the discrete courier task in \cite{mirowski2018learning}.
The random leader provides the follower with a richer set of possible movements compare to previous settings.
As shown in Figure \ref{fig:ltperform}, this allows the follower to track even more regular trajectories.
The follower is robust to changes in the initial position, as seen in Figure \ref{fig:ltperform}(c) where the follower starts outside the circle.
Average distance between the follower and the leader and the average reward are given in Table \ref{tab:leader-following}.

\begin{table}[h]
    \centering
    \caption{Random leader, Single Follower results.}
    \label{tab:leader-following}
    \begin{tabular}{|c||c||c|}
        \hline
               & Average Distance & Average Reward \\
        \hline
        Random & $0.0047$         & $-0.5027$      \\
        \hline
        Square & $0.0218$         & $-2.4007$      \\
        \hline
        Circle & $0.0394$         & $-4.3689$      \\
        \hline
    \end{tabular}
\end{table}

\begin{figure}[h]
    \centering
    \includegraphics[width=\linewidth]{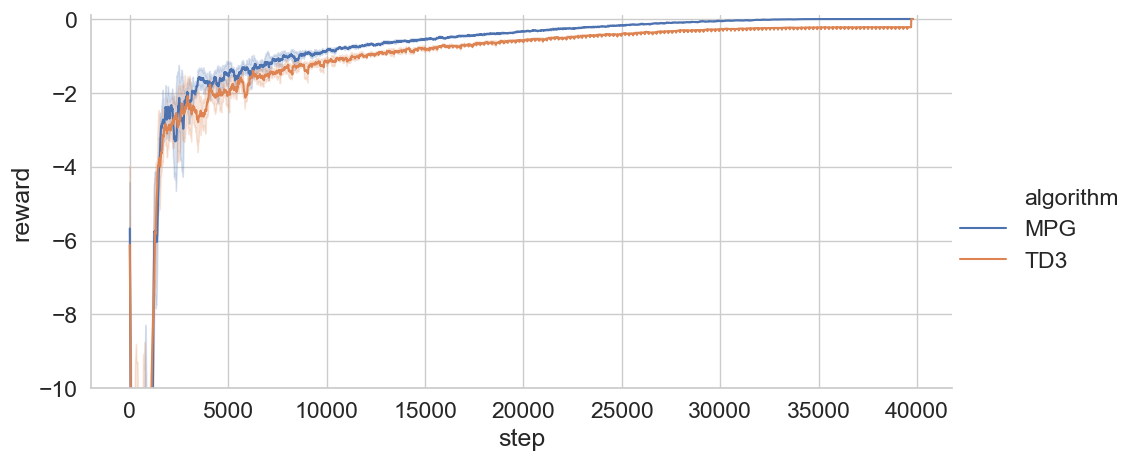}
    \caption{MPG vs. TD3 rewards during training for the circle leader-follower task.}
    \label{fig:reward-fl}
\end{figure}

\begin{figure}[h]
    \centering
    \subfigure[TD3]{
    \includegraphics[width=0.475\linewidth]{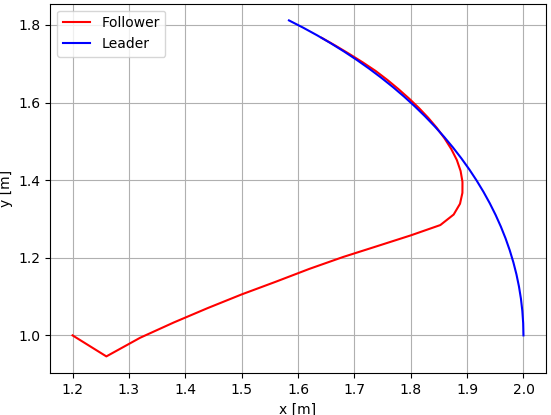}}
    \subfigure[MPG]{
    \includegraphics[width=0.475\linewidth]{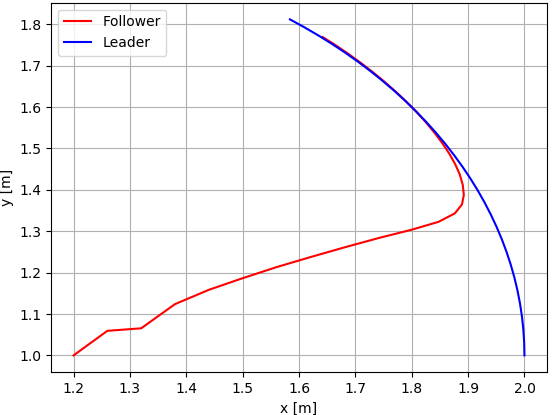}}
    \caption{Performance of TD3 and MPG followers trained with identical hyperparameters at 200 episodes.}
    \label{fig:TD3bad}
\end{figure}

For comparison, we also trained several followers using TD3.
An example reward is displayed in Figure \ref{fig:reward-fl}.
The values are smoothed using a 1D box filter of size 200.
For the circle leader task, TD3 which struggles to close the loop, slowly drifting away from the leader as time progresses.
The MPG trained agent does not suffer from this problem.

We also noticed that some TD3 trained followers do not move smoothly.
This is demonstrated in Figure \ref{fig:TD3bad}.
When trying to track a circle, these agents first dip down from $(1, 1.2)$ despite the leader moving counter-clockwise, starting at $(2, 1)$.

\subsection{Formation Control}

\begin{figure*}[h]
    \centering
    \subfigure[Rigid Unison with regular leader]{
    \includegraphics[width=0.3\textwidth]{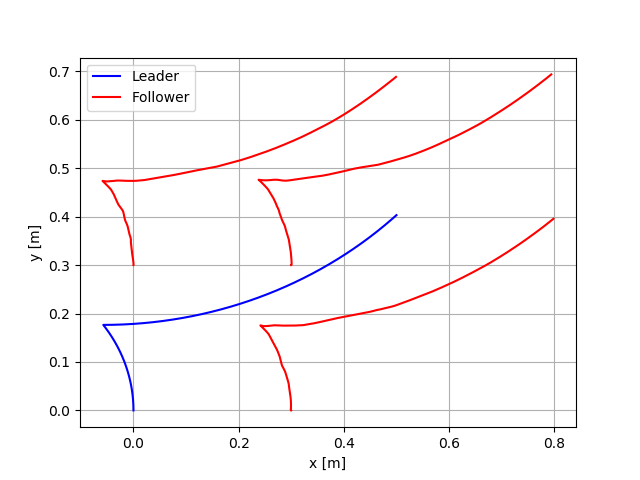}}
    \subfigure[Rigid Unison with random leader]{
    \includegraphics[width=0.3\textwidth]{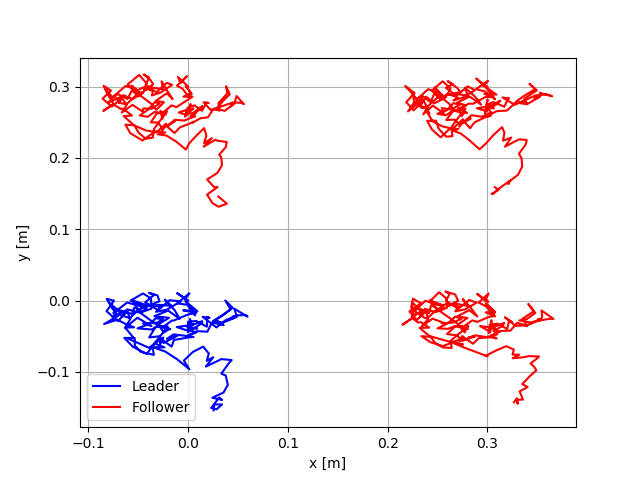}}
    \subfigure[Multi-agent Consensus]{
    \includegraphics[width=0.3\textwidth]{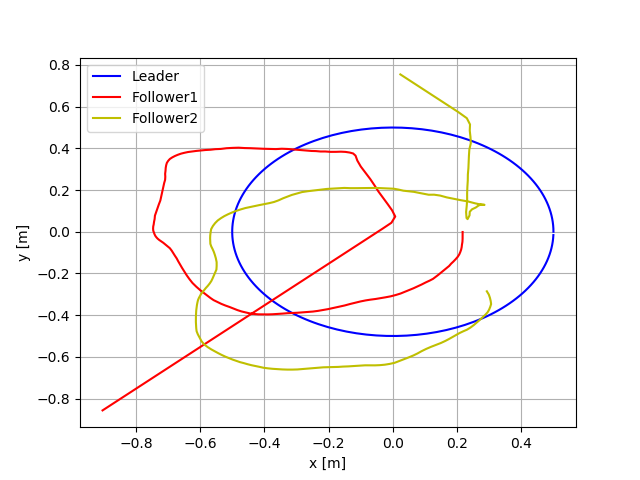}}
    \caption{Formation control: (a) and (b) use the unison definition while (c) allows the followers to swap positions.}
    \label{fig:fcperform}
\end{figure*}

We naturally extend from the tracking problem to displacement-based formation control by adding additional trained followers to track the sole leader.
However, additional work is needed to achieve collision avoidance among the followers.
Based on displacement-base formation control, we design and conduct two simulations similar to \cite{he2019leader}:
\begin{itemize}
    \item \textbf{Unison formation control}: There is a pre-defined formation $F$ with respect to global frame.
    All agents maintain a rigid formation throughout the entire movement process.
    The orientation of the formation does not change.
    \item \textbf{Consensus formation control}: Agents are given freedom to adapt rapidly while keeping the formation.
    The orientation of the formation can be changed.
\end{itemize}
In all simulations, a single neural network takes as input the positions of all the agents with respect to a global coordinate system and issues movement commands to all agents in the same coordinate system.

For unison formation control, groups of agents have a rigid formation that should not rotate.
Given the position of the leader, whose index is $1$, each follower $i$ should try to minimize its distance to an intended relative location $\mathcal{F}_i$ with respect to the leader while avoiding collisions.
As the leader moves, the expected positions $\mathcal{F}_i$ move in unison.
Let $C$ be the minimum safety distance, above which collision can be avoided.
The reward is
\begin{equation}
    r = -k_c n_c - \sum_{i \geq 1} \norm{p_i - \mathcal{F}_i}
\end{equation}
where $k_c$ is collision coefficient and $n_c$ is the number of collisions and $p_i$ is the position of agent $i$.
Upon any collisions, we reset the environment and start a new episode.

Then we explore unison formation control for a square formation with curved and random leader movements.
As seen in Figure \ref{fig:fcperform}, three followers move in unison equally well when tracking (a) a leader with smooth trajectory starting from lower left corner and (b) a random leader.
During training, we observed that adding more agents results in longer training time.
This is because adding an agent increases the state space and action space, and thus our input and output dimensions, by 2.
Training time grows approximately linearly in the number of agents.
This makes it intractable to train and deploy large formations in the hundreds or thousands of robots.
The average reward and distances are reported in Table \ref{tab:unison}.

\begin{table}[h]
    \centering
    \caption{Unison average reward and distances}
    \label{tab:unison}
    \begin{tabular}{|c||c||c||c||c|}
        \hline
        Pattern & Reward  & Dist. to $\mathcal{F}_2$ & Dist. to $\mathcal{F}_3$ & Dist. to $\mathcal{F}_4$ \\
        \hline
        Random  & $-0.7214$ & $0.0089$ & $0.0116$ & $-0.0134$ \\
        \hline
        Regular & $-0.5239$ & $0.0011$ & $0.0123$ & $-0.0026$ \\
        \hline
    \end{tabular}
\end{table}

\begin{table}[h]
    \centering
    \caption{Consensus average reward and distances}
    \label{tab:consensus}
    \begin{tabular}{|c||c||c||c|}
        \hline
        Reward & $\abs{d_{1, 2}-\mathcal{D}_{1,2}}$ & $\abs{d_{1, 3}-\mathcal{D}_{1,3})}$ & $\abs{d_{2, 3}-\mathcal{D}_{2,3}}$ \\
        \hline
        $-27.9942$ & $0.0219$   & $0.0203$     & $0.0275$     \\
        \hline
    \end{tabular}
\end{table}

Unlike unison formation control which dictates the individual motion of all agents, multi-agent consensus keeps a formation with respect to the local frame.
The formation can rotate with respect to the global frame.
The problem definition allows for switching and expansion within a given topology, as long as there are no collisions.
This can be beneficial.
For example, the agents may need to move further apart to avoid an obstacle or tighten their formation to fit through some passageway.
In general, agents $i, j$ should maintain a constant distance $\mathcal{D}_{i, j}$ to each other.
Letting $d_{i, j} = \norm{p_i - p_j}$, the reward function is given according to the following equation.
\begin{equation}\label{eqn:consensus_reward}
    r = -k_c n_c - \sum_{i, j} \abs{ d_{i, j} - \mathcal{D}_{i,j} \ }
\end{equation}

In our simulations, we trained $2$ follower agents to maintain triangle formation with a leader undergoing random motion.
As shown in Figure \ref{fig:fcperform} (c), we test the performance of the multi-agent consensus while having the leader traverse a counter-clockwise circular trajectory.
The leader starts at $(0.5,0.5)$, follower $1$ starts at a random lower left position, and follower $2$ starts at a random upper right position.
Initially, the three agents are very far away from each other but quickly formed a triangle when the leader reaches around $(0.1,0.5)$.
We observe that the followers swapped their local positions within the formation when the leader arrives $(-0.5, 0.0)$.
This is because the reward function is only interested in their relative distances, and the agents can maintain the formation with less total movement by switching their relative positions in the group.
Results are reported in Table \ref{tab:consensus}.

\begin{figure}[h]
    \centering
    \includegraphics[width=\linewidth]{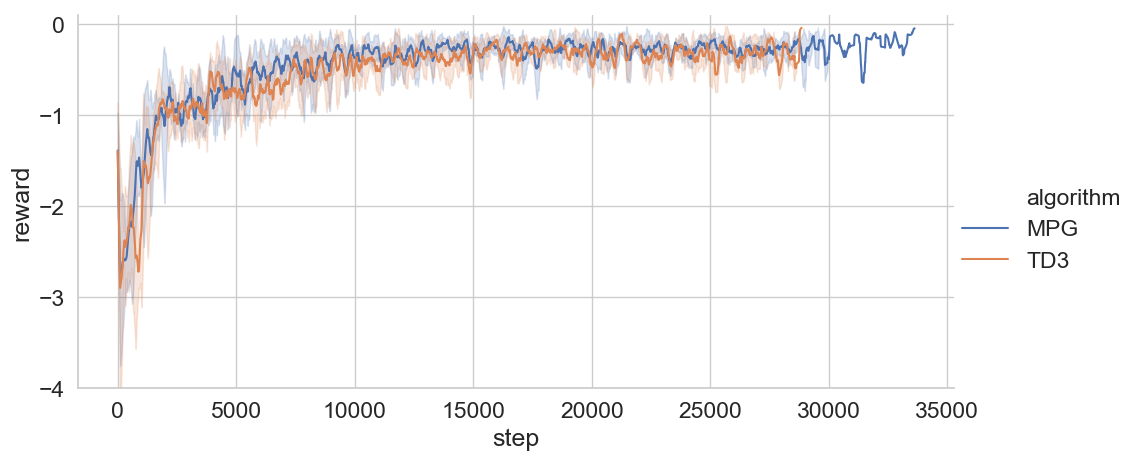}
    \caption{MPG vs. TD3 rewards during training for multi-agent consensus.}
    \label{fig:consensus_reward}
\end{figure}

For the purpose of comparison, an agent was also trained using TD3.
The rewards per time step are shown in Figure \ref{fig:consensus_reward}.
These were collected over 5 training runs of 200 episodes each.
The MPG curves are longer than the TD3 curves, because the MPG networks avoid episode ending collisions for longer.
Hence, MPG trained agents achieve the desired behavior sooner than the TD3 agents.

\subsection{Obstacle avoidance}

Based on the collision penalty $-k_c n_c$ embedded in Equation (\ref{eqn:consensus_reward}) of formation control, fixed or moving obstacle avoidance can naturally be integrated into the leader following or formation control problems.
Instead of control law redesign, obstacle avoidance can easily be achieved by adding an additional term in the reward function.

\begin{figure}[h]
    \centering
    \includegraphics[width=0.8\linewidth]{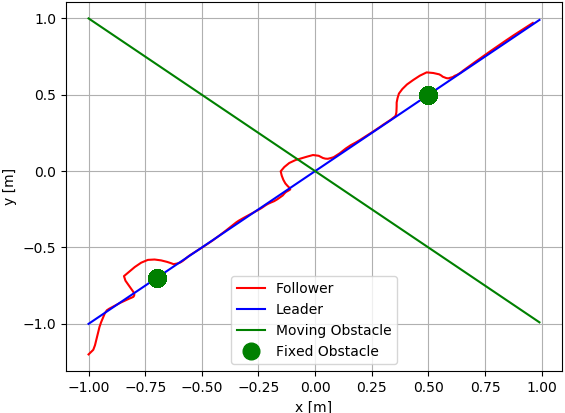}
    \caption{Follower avoiding the moving and fixed obstacles.}
    \label{fig:regulation}
\end{figure}

The simulation setup is illustrated in Figure \ref{fig:regulation}. 
The leader linearly travels from $(-1, -1)$ to $(1, 1)$.
We have $2$ fixed obstacles on the path of the leader that only stop the follower, and a moving obstacle travelling linearly from $(-1,1)$ to $(1,-1)$.
The follower, starting from a random location, must track the leader without hitting any obstacles.
The reward function is
\begin{equation}\label{eqn:reward_obs}
    r = - \norm{p_l - p_f} + k_o \sum_i \norm{p_f - o_i}
\end{equation}
where $o_i$ is the position of the obstacles and $k_o$ is the relative importance of avoiding the obstacles.

We use this reward, instead of a single penalty for colliding with an obstacle, because Equation (\ref{eqn:reward_obs}) gives constant feedback.
Training with sparse rewards is a big challenge in DRL \cite{salimans2018learning}.
In particular, because the obstacle is encountered later on, the follower learns to strictly copy the leader's movements without regard for the obstacles.
This is a local optimum that the agent fails to escape.
But for our purpose, it is not so important that the agent learns with even poorly shaped rewards.
The realism of the training settings is irrelevant as the agent is already in a artificial and simplified environment.

\section{Conclusion and Future Work}

In this paper, we propose a new DRL algorithm Momentum Policy Gradient to solve leader-follower tracking, formation control, and obstacle/collision avoidance problems, which are difficult to solve using traditional control methods.
These controllers can be trained in a simple toy environment, and then plugged into a larger modular framework.
The results show that MPG performs well in training agents to tackle a variety of continuous control tasks.
Image-based localization is achieved with a ResNet trained on a custom dataset.
Analysis demonstrates that the model can reliably predict a WMR's pose even when there are dynamic foreground/background, lighting, and view.

These methods are computationally inexpensive.
On a M40 Nvidia GPU and Intel Xeon E5 processor, MPG only takes, at most, $2$ hours to fully converge.
For localization, the CNN model took about $30$ hours to finish training on a dataset of over $9$ GB of images.
This is because the ResNet-50 model has $50$ residual blocks, compared to the $2$ layers in our DRL agents.

In the future, we would like to apply MPG to a wider variety of robotics problems.
In particular, we believe our current framework can generalize to larger indoor and outdoor settings.
Some work has already been done in this field \cite{bruce2018learning, bruce2017one}, but there are still major challenges that have not been solved.
One of these is data collection.
Currently, we require many samples to train our localization module.
However, using the techniques of few/one-shot learning \cite{bruce2017one} and data augmentation \cite{bruce2018learning}, it will no longer be necessary to collect so many samples at different times of day.
This opens up the possibility of online learning as a robot explores a new area of the environment.

\bibliographystyle{ieeetr}
\bibliography{IEEEabrv}

\end{document}